\documentclass[letterpaper, 12pt]{article}

\usepackage[utf8]{inputenc} % allow utf-8 input
\usepackage[T1]{fontenc}    % use 8-bit T1 fonts
\usepackage[colorlinks = true, linkcolor = blue, citecolor = blue]{hyperref}       % hyperlinks
\usepackage{url}            % simple URL typesetting
\usepackage{booktabs}       % professional-quality tables
\usepackage{amsfonts,amsmath,amsthm}       % blackboard math symbols
\usepackage{amssymb}
\usepackage{nicefrac}       % compact symbols for 1/2, etc.
\usepackage{fullpage}
\usepackage{color}
\usepackage{colordvi}
\usepackage{graphicx}
\usepackage{times}

\definecolor{Red}{rgb}{1,0,0}
\definecolor{Blue}{rgb}{0,0,1}
\definecolor{Olive}{rgb}{0.41,0.55,0.13}
\definecolor{Green}{rgb}{0,1,0}
\definecolor{MGreen}{rgb}{0,0.8,0}
\definecolor{DGreen}{rgb}{0,0.55,0}
\definecolor{Yellow}{rgb}{1,1,0}
\definecolor{Cyan}{rgb}{0,1,1}
\definecolor{Magenta}{rgb}{1,0,1}
\definecolor{Orange}{rgb}{1,.5,0}
\definecolor{Violet}{rgb}{.5,0,.5}
\definecolor{Purple}{rgb}{.75,0,.25}
\definecolor{Brown}{rgb}{.75,.5,.25}
\definecolor{Grey}{rgb}{.5,.5,.5}
\definecolor{Black}{rgb}{0,0,0}

\newtheorem{theorem}{Theorem}

\newtheorem{definition}{Definition}
\newtheorem{claim}{Claim}

\newcommand{\ds}{\delta}

\title{Expanding the class of global objective functions\\ for dissimilarity-based
	hierarchical clustering}

\author{
  Sebastien Roch\\
  Department of Mathematics\\
  University of Wisconsin--Madison\\
  \texttt{roch@math.wisc.edu}
}

\begin{document}

\maketitle

\begin{abstract}
  Recent work on dissimilarity-based hierarchical clustering has led to the introduction of global objective functions for this classical problem. 
  Several standard approaches, such as average linkage,
  as well as some new heuristics have been shown to provide approximation guarantees. 
  Here we introduce a broad new class of objective functions
  which satisfy desirable properties studied in prior work. 
  Many common agglomerative and divisive
  clustering methods are shown to be greedy algorithms for these objectives, which are inspired
  by related concepts in phylogenetics. 
\end{abstract}

\section{Introduction}
\label{sec:intro}

\paragraph{Background}
In hierarchical clustering, one seeks a recursive partitioning
of the data that captures clustering information at different
levels of granularity. Classical work on the subject
mostly takes an algorithmic perspective. In particular,
various iterative clustering methods have been developed, including the well-known bottom-up dissimilarity-based approaches
single linkage, average linkage, etc.~(see, e.g.,~\cite[Chapter 25]{murphy_machine_2012}).
Recent work on dissimilarity-based hierarchical clustering 
has emphasized a different, optimization-based, perspective. This 
has led to the introduction of global objective functions for this classical problem~\cite{Dasgupta:16}. 
Some standard approaches
as well as new heuristics have been shown to provide approximation guarantees~\cite{Dasgupta:16,NIPS2016_6325,CharikarChatziafratis:17,CKMM:18,chatziafratis_hierarchical_2018,charikar_hierarchical_2019,alon_hierarchical_2020}. %\todo{Probably some new ones.}
These new objective functions have also been justified through their behavior on random or structured input models~\cite{Dasgupta:16,NIPS2017_7200,CKMM:18,manghiuc_hierarchical_2021}. 
%\todo{Say more about what was done. Maybe say why this is a useful (without overselling it) perspective (quote Dasgupta; constraints, active?). Start by saying that it opened up a connection to approximation algorithms, which has been active since; but also has other advantages.}

Here we introduce a broad new class of objective functions
which satisfy natural, desirable properties considered in previous work. These new objectives are inspired by related work in the phylogenetic reconstruction literature (see Section~\ref{sec:heuristic}). We argue that several common agglomerative and divisive
clustering methods, including average linkage, single linkage and recursive sparsest cut, can be interpreted as greedy algorithms for these objectives. 
%\todo{Say something about random imput illustrations. Also emphasize connection to mathematical biology, and cite enough papers about it. May want to say that our focus in this work is not computational, although mention greedy algorithms.}

\paragraph{Definitions and main results}
Our input data is a collection of $n$ objects to be clustered,
which we denote without loss of generality $L := \{1,\ldots,n\}$,
together with a dissimilarity map.\footnote{Our results can also be adapted to the case where the input are \emph{similarities}. Throughout, we 
confine ourselves to dissimilarities for simplicity.}
\begin{definition}[Dissimilarity]
A \emph{dissimilarity} on $L$ is a map
$\ds \,:\, L\times L \to [0,+\infty)$ which satisfies: $\ds(x,x) = 0$ for all $x$ and
$\ds(x,y) = \ds(y,x) > 0$ for all $x\neq y$.
\end{definition}
\noindent For disjoint subsets $A, B \subseteq L$, we overload $\ds(A,B) := \sum_{x\in A, y\in B} \ds(x,y)$
and define
$$
\bar{\ds}(A,B) = \frac{\ds(A,B)}{|A| |B|}.
$$ 

As in~\cite{Dasgupta:16,CKMM:18}, we encode a hierarchical clustering as a rooted binary tree whose leaves are the objects to be clustered.
\begin{definition}[Hierarchy]
A \emph{hierarchy} on $L$ is a rooted binary tree $T = (V,E)$ with $n$ leaves, which we identify with the set $L$. 
\end{definition}
\noindent We will need some notation.
The leaves \emph{below} $v \in V$, i.e., of the subtree $T[v]$ rooted at $v$, are denoted by $L_T[v] \subseteq L$.
%We refer to the elements of $C_T := \{L[v]:v \in V\}$ as \emph{clusters}. 
We let $S_T$ be the internal vertices of $T$, that is, its non-leaf vertices. 
For $s \in S_T$, we denote by $s_-$ and $s_+$ the immediate descendants of $s$ in $T$. 
%Then $s$ corresponds to the \emph{split} $(L[s_-],L[s_+])$.
For a pair of leaves $x\neq y \in L$, the most recent common ancestor of $x$ and $y$ in $T$, denoted $x\land_T y$,
is the internal vertex $s$ furthest from the root (in graph distance) such that $x, y \in L_T[s]$. %\todo{Make sure this notation change is done throughout.}
% Finally, we let $\mathbb{T}_L$ be the collection of all hierarchies on $L$.
%Note that the collection $\{L[v]:v \in V\}$ is \emph{nested}: for any $u, u' \in V$, it holds that $L[u] \cap L[u'] = \emptyset$, $L[u] \subseteq L[u']$ or $L[u'] \subseteq L[u]$. 
%\todo{Will probably need the collection of all hierarchies, both of each size, and of general size? Or do I fix the problem size? Since we want subsets, might be useful to be more general. Also the objective doesn't care about size? Maybe should explicitly use $L = [n]$?}

In our setting, the goal of hierarchical clustering is to map dissimilarities to hierarchies. As a principled way to accomplish
this, a global objective function over hierarchies was proposed
in~\cite{Dasgupta:16}. %\todo{Define what an objective is. And what the optimization problem is. Measure the mismatch between a hierarchy and a dissimilarity on the same set? This is important since we're defining a class of functions here. Also, what some natural/desirable properties of these functions?} 
It was generalized in~\cite{CKMM:18}
to objectives %\todo{From what to what? Be specific about what an objective should be here.} 
of the form:
\begin{equation}\label{eq:dasgupta-cost}
\Gamma(T;\ds)
= \sum_{s \in S_T} \gamma(|L_T[s_-]|,|L_T[s_+]|) \,\ds(L_T[s_-],L_T[s_+]),
\end{equation}
where $\gamma$ is a given real-valued function and $|A|$ is the number of elements in $A$. 
One then seeks a hierarchy $T$ which minimizes $\Gamma( \,\cdot\,;\ds)$ under input $\ds$.\footnote{Note that we deviate from~\cite{Dasgupta:16,CKMM:18} (in the \emph{dissimilarity} setting) and \emph{minimize} the objective function.}
For instance, the choice
$\gamma(a,b) = \gamma_\mathrm{D}(a,b) := -a-b$ was made in~\cite{Dasgupta:16}.

It was shown in~\cite{Dasgupta:16,CKMM:18} that, for $\gamma_\mathrm{D}$, the objective $\Gamma$ satisfies several natural conditions. In particular, it satisfies the following.
\begin{definition}[Unit neutrality]
An objective $\Gamma$ is \emph{unit neutral} if:
all hierarchies 
have the same cost under the \emph{unit dissimilarity} $\ds(x,y) = 1$ for all $x\neq y$.
\end{definition}
\noindent Moreover, this $\Gamma$ behaves well %\todo{The word consistent here may confuse the statisticians.} 
on ultrametric inputs. Formally,
a dissimilarity $\ds$ on $L$ is an \emph{ultrametric} if
for all $x,y,z \in L$, it holds that
\begin{equation}\label{eq:ultrametric}
\ds(x,y) \leq \max\{\ds(x,z),\ds(y,z)\}.
\end{equation}
Ultrametrics are naturally associated to hierarchies in the following
sense. If $\ds$ is an ultrametric, then there is a (not necessarily unique) hierarchy $T$ together with a height function
$h : S_T \to (0,+\infty)$ such that, for all $x\neq y \in L$, it holds that 
\begin{equation}\label{eq:dist-height}
\ds(x,y) = h(x\land_T y). 
\end{equation}
See e.g.~\cite{SempleSteel:03} for details. We say that such a hierarchy $T$ is
\emph{associated} to ultrametric $\ds$.
\begin{definition}[Consistency on ultrametrics\footnote{Our definition is related to what is referred to as \emph{admissibility} in~\cite{CKMM:18}. We will not introduce the more general setting of~\cite{CKMM:18} here. %\todo{Be specific on the difference.}
}%\todo{Change all to admissible and make sure it's consistent with previous use. Faithful?}
]
The objective function $\Gamma$ is \emph{consistent on ultrametrics} if the following holds for all ultrametric $\ds$ and associated hierarchy $T$. For any hierarchy $T'$, we have the inequality
\begin{displaymath}
\Gamma(T;\ds)
\leq \Gamma(T';\ds).
\end{displaymath}
In other words, a hierarchy associated to an ultrametric $\ds$ is a global minimum under input $\ds$.
\end{definition}
%\noindent\todo{What about~\eqref{eq:dasgupta-cost}? Say that these are minimal (really?) assumptions made and that more functions will satisfy them? What would be a broad, natural class?}
\noindent Observe that unit neutrality in fact follows %\todo{Just make this whole thing into a remark then.}
from consistency on ultrametrics
as the unit dissimilarity on $L$ is an ultrametric 
that can be realized
on \emph{any} hierarchy by assigning height $1$ to all internal vertices.

Here we introduce a broad new class of global objective functions for dissimilarity-based hierarchical clustering. 
For a subset of pairs $M \subseteq L\times L$, let $\ds|_M:M \to [0,+\infty)$ denote the
dissimilarity $\ds$ restricted to $M$, i.e., $\ds|_M(x,y) = \ds(x,y)$ for all $(x, y) \in M$. Let also $\min \ds|_M$ and
$\max \ds|_M$ be respectively the minimum and maximum value of $\ds$ over pairs in $M$. We consider objective functions of the general form %\todo{Would it be easier to have $\delta$ as an input to $\hat{h}$?}
\begin{equation}\label{eq:new-cost}
\Gamma(T;\ds)
= \sum_{s \in S_T} 
\hat{h}(T[s],\ds|_{L_T[s_-]\times L_T[s_+]}),
\end{equation}
where we require moreover that 
the function $\hat{h}$ %\todo{What is this ``function'' defined over?} 
satisfy the condition
\begin{equation}\label{eq:hhat-condition}
\hat{h}(T[s],\ds|_{L_T[s_-]\times L_T[s_+]})
\in
\left[
\min \ds|_{L_T[s_-]\times L_T[s_+]},
\max \ds|_{L_T[s_-]\times L_T[s_+]}
\right],
\end{equation}
for any hierarchy $T$, any $s \in S_T$ and any dissimilarity $\ds$. 
%\todo{For every? Or something more generic?}
We refer to such an objective function as a \emph{length-based
	objective}, a name which will be explained in Section~\ref{sec:new-cost}. %\todo{Should probably mention the connection to mathematical biology at this point. Maybe say that it associates to each internal vertex a ``summary dissimilarity'' and takes a sum.}
%We will explain the name below, but for now we mention only that $\hat{h}$ is meant to be an estimate of the height of the corresponding internal vertex (see our discussion of ultrametrics above). Many natural examples will be given in Section~\ref{sec:new-cost}.

One possibility for $\hat{h}$ is
\begin{equation}\label{eq:dbar}
\hat{h}(T[s],\ds|_{L_T[s_-]\times L_T[s_+]})
= \bar{\ds}(L_T[s_-],L_T[s_+])
= 
\frac{\ds(L_T[s_-],L_T[s_+])}{|L_T(s_-)|\,|L_T(s_+)|},
\end{equation}
which is indeed a function of only $T[s]$ and $\ds|_{L_T[s_-]\times L_T[s_+]}$. This choice is a special case of~\eqref{eq:dasgupta-cost} with 
\begin{displaymath}
\gamma(|L_T[s_-]|,|L_T[s_+]|) = \frac{1}{|L_T[s_-]|\,|L_T[s_+]|}.
\end{displaymath}
We show in Section~\ref{sec:heuristic} that there
are many other natural possibilities that do not
fit in the framework~\eqref{eq:dasgupta-cost},
including more ``non-linear'' choices.
%(In fact, this is the only choice of the form~\eqref{eq:dasgupta-cost} that also satisfies~\eqref{eq:new-cost} and~\eqref{eq:hhat-condition}
%for any $\ds$ and $T$. \todo{This sentence should be extended---or removed. In particular, point out that our definition is not a generalization of~\eqref{eq:dasgupta-cost}.})
%On the unit dissimilarity, we have $\bar{\ds}(L_T[s_-],L_T[s_+]) = 1$
%and, hence the resulting objective is unit neutral since all
%rooted binary trees on $n$ leaves have the same number of 
%internal vertices (see e.g.~\cite{SempleSteel:03}). 
%It also follows from a result\footnote{Translated to our \emph{minimization} setting for \emph{dissimilarities}.} in~\cite{CKMM:18} \todo{Which one?} that
%this particular choice of objective is consistent on ultrametrics.

%\todo{Should say at this point that many more functions, in particular more ``non-linear'' ones, are in this new class. In fact, maybe already want to give some examples, to give an idea.}
%More generally, we establish in Section~\ref{sec:new-cost}
%the following result.
Our main result is that, under condition~\eqref{eq:hhat-condition}, our new objectives are unit neutral and consistent on ultrametrics and therefore provide sound global objectives for hierarchical clustering.
\begin{theorem}[Length-based objectives]\label{thm:main}
Any length-based objective satisfying~\eqref{eq:hhat-condition}
is unit neutral and consistent on ultrametrics.
\end{theorem} 

\paragraph{Organization} The rest of the paper is organized as follows. Theorem~\ref{thm:main} is proved in Section~\ref{sec:proof}. Motivation and further related work
is provided in Section~\ref{sec:heuristic}.

%\noindent\todo{Probably want to state and prove the strict version later?} \todo{What about: We show in Section~\ref{sec:heuristic} that 
%average linkage and recursive sparsest cut are, in fact, greedy
%algorithms for this objective. ALSO: this is where we want to mention that our focus is not computational, but it would be interesting to develop approximation algorithms for these new functions.}
%\todo{Put this somewhere: We will explain the name below, but for now we mention only that $\hat{h}$ is meant to be an estimate of the height of the corresponding internal vertex (see our discussion of ultrametrics above). Many natural examples will be given in Section~\ref{sec:new-cost}.}
%
%
%\paragraph{Related work}
%As explained in Section~\ref{sec:heuristic},
%our main definition is inspired by the concept
%of minimum evolution in phylogenetics. See e.g.~\cite{GascuelSteel:06}, and references therein, for a fascinating account. Consult e.g.~\cite{SempleSteel:03} for more
%background on phylogenetics. \todo{Many more references to phylogenetics.}
%\todo{Should also cite the theoretical computer science work: Balcan and company.}
%
%\todo{Should point out that this idea of an objective function for hierarchical clustering is not new. In phylogenetics, BME is doing exactly that.}
%
%
%\todo{Should also mention Roch:10 in introduction and in related work.} \todo{Should say something more explicit about WPGMA. Also connect this to NJ, because phylogenetics reviewers might object WPGMA is not a good algorithm.}

%\section{Average linkage and recursive sparsest cut: what is the objective?}
\section{Motivation}
\label{sec:heuristic}
\label{sec:new-cost}

%\todo{Fix the labels of this section (and proof one).}

To motivate our class of objectives for hierarchical clustering, we first give a heuristic derivation of the choice~\eqref{eq:dbar}, which is inspired by the concept of minimum evolution (see e.g.~\cite{GascuelSteel:06} and references therein). In phylogenetics,
one is given molecular sequences from extant species and the goal is to reconstruct a phylogenetic tree
representing the evolution of these species (together with edge lengths which roughly measure the amount of evolution on the edges). One popular approach is to estimate a distance between each pair of species by comparing their molecular sequences. Various distance-based methods exist. One such class of methods relies on the concept
of minimum evolution, which in a nutshell stipulates that the best tree is the shortest one (i.e., the one with the minimum sum of edge lengths). Put differently, in the spirit of Occam's razor, the solution involving the least amount of evolution to explain the data should be preferred. Without going into details (see, e.g.,~\cite{SempleSteel:03,Steel:16,Warnow:17} for comprehensive introductions to phylogenetic reconstruction methods), methods based on minimum evolution have been highly successful in practice. In particular, one of the most popular methods in this area is Neighbor-Joining~\cite{NJ}, which has been shown to be a greedy method~\cite{GascuelSteel:06} for a variant of minimum evolution called balanced minimum evolution. Below we establish a formal connection to our class of objectives.

\paragraph{Total length}
Inspired by the concept of minimum evolution, we reformulate the length-based objective with choice~\eqref{eq:dbar}, i.e.,
\begin{equation}\label{eq:ls-cost}
\Gamma(T;\ds)
= \sum_{s \in S_T} 
\bar{\ds}(L_T[s_-],L_T[s_+]),
\end{equation}
as an estimate of the ``total length of the hierarchy $T$ under $\ds$.'' %\todo{Be more explicit that this is where it's coming from. Pauplin has a good summary of the two main criteria used in phylogenetics to ``fit'' a dissimilarity.} 
We start with the ultrametric case.
If $\ds$ is ultrametric
and $T$ is associated to $\ds$ with height function $h$
then, for any $s \in S_T$, $x\in L_T[s_-]$ and $y \in L_T[s_+]$,
we have
\begin{equation}\label{eq:derivation-dsbar}
\bar{\ds}(L_T[s_-],L_T[s_+])
= \ds(x,y) = h(s).
\end{equation}
Moreover, letting $M = \max \ds + 1$, 
consider a modified rooted tree $\tilde{T} = (\tilde{V},\tilde{E})$ with an extra edge
connected to the root of $T$ and associate height $M$ to the new root so created. Then assign to each edge $e = (s_0,s_1)$ of $\tilde{T}$ a length equal to $h(s_0) - h(s_1)$, where $s_0$ is closer to the root than $s_1$. Then the \emph{total length} of $\tilde{T}$ 
is
\begin{equation*}
\sum_{e = (s_0,s_1) \in \tilde{E}} [h(s_0) - h(s_1)]
= M + \sum_{s \in S_T} h(s) 
= M + \sum_{s \in S_T} 
\bar{\ds}(L_T[s_-],L_T[s_+])
= M + \Gamma(T;\ds),
\end{equation*}
where we used the fact that each non-root internal vertex of $\tilde{T}$ is counted 
twice positively and once negatively (since it has two immediate children and one immediate parent), while the root of $\tilde{T}$ is counted once. In other words, up to translation by $M$, 
$\Gamma(T;\ds)$ measures the total length of hierarchy $T$
associated to ultrametric $\ds$. 

More generally, on a heuristic level, if $\ds$ is not ultrametric (but perhaps close to one) and $T$ is any hierarchy we interpret 
$\bar{\ds}(L_T[s_-],L_T[s_+])$ as an %least-squares 
estimate of the height of $s$ on $T$ based on the values $\ds|_{L_T[s_-]\times L_T[s_+]}$. %(which would be a reasonable approximation if one were to assume that $\ds$ is close to an ultrametric associated to $T$). 
Then we see $\Gamma(T;\ds)$ as an estimate
of the total length\footnote{Note that we are not imposing that estimated edge lengths be positive.} of $T$ under $\ds$. Minimizing $\Gamma$ hence corresponds roughly speaking to finding a hierarchy whose estimated total length is minimum under a fit to the input $\ds$. 
%It may not be immediately obvious why this is a good criterion for hierarchical clustering, but for one 
In addition to its connection to the fruitful concept of minimum evolution in phylogenetics, 
as pointed out in Section~\ref{sec:intro} this objective has the desirable property of being
consistent on ultrametrics. 
%Further, in the phylogenetic context where lengths of edges measure amount of evolution, minimizing the total length of the output hierarchy (called phylogenetic diversity there) corresponds to finding a parsimonious explanation of the data, that is, a hierarchy that is consistent with the least amount of evolution when fitted to an observed dissimilarity.\footnote{In phylogenetics, this is done in the more general context of additive metrics, rather than ultrametrics. See e.g.~\cite{SempleSteel:03}.}

\paragraph{Other choices for $\hat{h}$}
%\todo{Incorporate the following above (start with general case).} 
%In the argument presented in Section~\ref{sec:heuristic},
%the heights of the internal vertices were estimated by least-squares.
Interpreting $\hat{h}$ as a height estimator suggests many more natural choices.
For instance, one can take a model-based 
approach such as the one advocated in the related work of~\cite{Degens:83,CastroNowak:04}. 
There, a simple error model is assumed (adapted to our setting):
the dissimilarity $\ds$ is in fact an ultrametric $\ds^*$ plus an entrywise additive
noise that is i.i.d. If $T$ is associated to $\ds^*$ and
$s \in S_T$, then a likelihood-based estimate of $h(s)$
can be obtained from the values
$\ds|_{L_T[s_-]\times L_T[s_+]}$, which all share the same mean $h(s)$ and are independent. Under the assumption that the
additive noise is Gaussian for instance, one recovers the least-squares estimate~\eqref{eq:dbar}. 
Taking the noise to be Laplace leads to the median.
As pointed out by~\cite{Degens:83},
other choices also lead to estimates that arise naturally in the hierarchical clustering
context.  If the density of the noise is assumed to be $0$ below $0$ and non-increasing above $0$, then the maximum likelihood estimate is the minimum of the observed values $\ds|_{L_T[s_-]\times L_T[s_+]}$. Note that all these examples satisfy condition~\eqref{eq:hhat-condition}
and therefore Theorem~\ref{thm:main} implies
that they produce length-based objectives
\begin{equation*}
	\Gamma(T;\ds)
	= \sum_{s \in S_T} 
	\hat{h}(T[s],\ds|_{L_T[s_-]\times L_T[s_+]}),
\end{equation*}
that are consistent
on ultrametrics. 

%\paragraph{Other connections}
We note further that we allow in general the function $\hat{h}$
to depend on the structure of the subtree
rooted at the corresponding internal vertex. For instance,
one could consider a weighted average of the quantities
$\ds|_{L_T[s_-]\times L_T[s_+]}$ where the weights depend on the
graph distance between the leaves. In the phylogenetic
context, precisely such an estimator, namely 
\begin{equation}\label{eq:motiv-science}
	\hat{h}(T[s],\ds|_{L_T[s_-]\times L_T[s_+]})
	= \sum_{x\in L_T[s_-], y\in L_T[s_+]} 2^{-|x|_{L_T[s_-]}} 2^{-|y|_{L_T[s_+]}} \ds(x,y),
\end{equation}
where $|x|_{L_T[s_-]}$ denotes the graph distance between $s_-$ and $x$ in $T[s]$,
has been shown 
rigorously to lead to significantly better height estimates in certain regimes of parameters for standard models of sequence evolution~\cite{Roch:10}. The analysis
of this estimator accounts for the fact that
the dissimilarities in $\ds|_{L_T[s_-]\times L_T[s_+]}$ are not independent---but in fact highly correlated---under these models.
It can be shown (by induction on the size of the hierarchy) that $\sum_{x\in L_T[s_-], y\in L_T[s_+]} 2^{-|x|_{T[s]}} 2^{-|y|_{T[s]}} = 1$, and therefore Theorem~\ref{thm:main}  applies in this case as well.

\paragraph{Greedy algorithms}
Finally, we connect our class of objectives to 
standard approaches to hierarchical clustering. 
%We focus on the case of~\eqref{eq:dbar}. 
The first clustering approach we consider, average linkage, is an agglomerative method. 
\begin{enumerate}
	\item[0.] \textit{Average linkage}
	
	\item[1.] Input: dissimilarity $\ds$ on $L = \{1,\ldots,n\}$.
	
	\item[2.] Create $n$ singleton trees with leaves respectively $1,\ldots, n$.
	
	\item[3.] While there are at least two trees left:
	
	\begin{enumerate}
		\item[a-] Pick two trees $T_1, T_2$ with leaves $A_1, A_2$ \textbf{minimizing} $\bar{\ds}(A_1,A_2)$.
		
		\item[b-] Merge $T_1$ and $T_2$ through a new common root adjacent to their roots.
	\end{enumerate}
	
	\item[4.] Return the resulting tree.
\end{enumerate}
The second method we consider, recursive sparsest cut, is a divisive method.
\begin{enumerate}
	\item[0.] \textit{Recursive sparsest cut}
	
	\item[1.] Input: dissimilarity $\ds$ on $L = \{1,\ldots,n\}$.
	
	\item[2.] Find a partition $\{A_1, A_2\}$ of $L$ \textbf{maximizing} $\bar{\ds}(A_1,A_2)$.
	
	\item[3.] Recurse on $\ds|_{A_1 \times A_1}$ and $\ds|_{A_2 \times A_2}$ to obtain trees $T_{A_1}$ and $T_{A_2}$.
	
	\item[4.] Merge $T_{A_1}$ and $T_{A_2}$ through a new common root adjacent to their roots.
	
	\item[5.] Return the resulting tree.
\end{enumerate}
Note that Step (2) is NP-hard and one typically resorts to approximation algorithms~\cite{Dasgupta:16}.
%\todo{Emphasize that we do not prove anything about the algorithms here. But mention whatever results are known. Maybe put this below.}

We argue here that both these methods are greedy algorithms
for the \emph{same} global objective function. From an algorithmic point of view, these
methods proceed in an intuitive manner: average linkage 
starts from the bottom and iteratively merges clusters that
are as similar as possible according to $\bar{\ds}$; 
recursive sparsest cut starts from the top and iteratively 
splits clusters that are as different as possible according
to $\bar{\ds}$. From an optimization point of view, both methods seemingly use the same local criterion: $\bar{\ds}$. But, given that at each iteration one \emph{minimizes} while the other \emph{maximizes} this criterion, what is their \emph{common global objective}? 

%\todo{
%Put somewhere here: In what sense, then, are
%average linkage and recursive sparsest cut \emph{greedy algorithms} for this common global objective?}
We claim it is~\eqref{eq:ls-cost}.
At each iteration, average linkage forms a new cluster whose contribution to~\eqref{eq:ls-cost} is minimized among
all possible merging choices.
As for recursive sparsest cut: when splitting $A_1$ and $A_2$,
the value $\bar{\ds}(A_1,A_2)$ is (in the interpretation above)
the estimated height of the parent $s_{A_1,A_2}$ of the two corresponding
subtrees; by maximizing $\bar{\ds}(A_1,A_2)$, one then
greedily \emph{minimizes} the length of the newly added edge \emph{above}
$s_{A_1,A_2}$ and, hence, the contribution of that edge
to the total length.

%\todo{Also mention Rob's thing.}

%\todo{The next part should be moved to greedy algorithms.}
%Plugging this estimate for the height
%of the corresponding internal vertex in~\eqref{eq:new-cost}
%and applying the bottom-up heuristic of Section~\ref{sec:heuristic} produces single linkage. 
Other choices of $\hat{h}$ lead to single linkage, complete linkage and median linkage as well as the more general agglomerative approach
of~\cite{CastroNowak:04}.
For instance, single linkage greedily minimizes the
estimated total length of a hierarchy whose heights
are fitted using maximum likelihood assuming the noise has any density that is $0$ below $0$ and is non-increasing above $0$. The choice~\eqref{eq:motiv-science} on the other hand leads to WPGMA. %~\cite{Roch:10}.

%\section{Examples}

\section{Proof of Theorem~\ref{thm:main}}
\label{sec:proof}

In this section, we prove Theorem~\ref{thm:main}.
As we noted above, it suffices to prove consistency on ultrametrics, as it implies unit neutrality.
%\begin{proposition}[Length-based objectives are unit neutral]
%Let $\Gamma$ be a length-based objective
%with corresponding $\hat{h}$ function. 
%Then $\Gamma$ is unit neutral.
%, that is, \todo{Do we really need to repeat?} for the unit dissimilarity $\ds(x,y) = 1$ for all $x\neq y \in L$, it holds that $\Gamma(T;\ds) = \Gamma(T';\ds)$ for any distinct pair of hierarchies $T$ and $T'$.
%\end{proposition}
%\begin{proof}
%For the unit dissimilarity $\ds$, by condition~\eqref{eq:hhat-condition} we have
%\begin{equation*}
%\hat{h}(T[s],\ds|_{L[s_-]\times L[s_+]})
%= 1,
%\end{equation*}
%for all hierarchies $T$ and all $s \in S_T$. Hence, 
%by definition of $\Gamma$, it holds that
%\begin{equation*}
%\Gamma(T;\ds)
%= \sum_{s \in S_T} 
%\hat{h}(T[s],\ds|_{L[s_-]\times L[s_+]})
%= |S_T|
%= n-1.
%\end{equation*}
%See e.g.~\cite{SempleSteel:03}.
%Observe that unit neutrality in fact follows %\todo{Just make this whole thing into a remark then.}
%from consistency on ultrametrics (proved below)
%as the unit dissimilarity on $L$ is an ultrametric 
%that can be realized
%on \emph{any} hierarchy by assigning height $1$ to all internal vertices.
%\end{proof}
%It remains to establish consistency on ultrametrics.
%\todo{Should probably comment on previous proofs.}
%\begin{proposition}[Length-based objectives are consistent on ultrametrics]\label{prop:consistent}%\todo{Theorem: Main result}
%	Let $\Gamma$ be a length-based objective
%	with corresponding $\hat{h}$ function. %\todo{Move this to the proof?}
%	Then $\Gamma$ is consistent on ultrametrics.
%\end{proposition}
\begin{proof}[Proof of Theorem~\ref{thm:main}]
We first reduce the proof to a special $\hat{h}$. 	
\begin{claim}[Reduction to minimum]
It suffices to prove 
Theorem~\ref{thm:main}
%Proposition~\ref{prop:consistent} 
for the choice $\hat{h} = \hat{h}_m$ where
\begin{equation}\label{eq:hhat-special}
\hat{h}_m(T[s],\ds|_{L_T[s_-]\times L_T[s_+]})
=
\min \ds|_{L_T[s_-]\times L_T[s_+]}.
\end{equation}
\end{claim}
\begin{proof}
Let $\hat{h}$ be an arbitrary choice of function  satisfying~\eqref{eq:hhat-condition} and let $\ds$
be an ultrametric with associated hierarchy $T$. Recall that we seek to show that $\Gamma(T;\ds) \leq \Gamma(T';\ds)$ for any hierarchy $T'$.

By~\eqref{eq:dist-height}, for any $s \in S_T$ and for any $x, x' \in L_T[s_-]$ and $y, y' \in L_T[s_+]$, we have
\begin{equation*}
\ds(x,y) = \ds(x',y') = \min \ds|_{L_T[s_-]\times L_T[s_+]}
= \max \ds|_{L_T[s_-]\times L_T[s_+]},
\end{equation*}
since $x\land_T y = x' \land_T y' = s$, where the first equality
over all choices of $x,x'y,y'$ implies the other two. 
Therefore, under the ultrametric associated to $T$, 
this arbitrary $\hat{h}$ in fact satisfies
\begin{equation*}
\hat{h}(T[s],\ds|_{L_T[s_-]\times L_T[s_+]})
=
\min \ds|_{L_T[s_-]\times L_T[s_+]},
\end{equation*}
by condition~\eqref{eq:hhat-condition}. This holds for all $s \in S_T$, so that 
\begin{equation}\label{eq:reduction-1}
\Gamma(T;\ds)
= \sum_{s \in S_T} 
\hat{h}(T[s],\ds|_{L_T[s_-]\times L_T[s_+]})
= \sum_{s \in S_T} 
\hat{h}_m(T[s],\ds|_{L_T[s_-]\times L_T[s_+]}),
\end{equation}
takes the same value for any $\hat{h}$.

On the other hand,
for any other hierarchy $T'$ and for any internal vertex $s' \in S_{T'}$ it holds that
\begin{equation*}
\hat{h}(T'[s'],\ds|_{L_{T'}[s'_-]\times L_{T'}[s'_+]})
\geq
\min \ds|_{L_{T'}[s'_-]\times L_{T'}[s'_+]},
\end{equation*}
by condition~\eqref{eq:hhat-condition}. Hence,
\begin{equation}\label{eq:reduction-2}
	\Gamma(T';\ds)
= \sum_{s' \in S_{T'}} 
\hat{h}(T'[s],\ds|_{L_{T'}[s_-]\times L_{T'}[s_+]})
\geq \sum_{s' \in S_{T'}} 
\hat{h}_m(T'[s'],\ds|_{L_{T'}[s'_-]\times L_{T'}[s'_+]}).
\end{equation}

Combining~\eqref{eq:reduction-1} and~\eqref{eq:reduction-2}, we see that establishing the
desired inequality under the choice~\eqref{eq:hhat-special} 
$$
\sum_{s \in S_T} 
\hat{h}_m(T[s],\ds|_{L_T[s_-]\times L_T[s_+]})
\leq \sum_{s' \in S_{T'}} 
\hat{h}_m(T'[s'],\ds|_{L_{T'}[s'_-]\times L_{T'}[s'_+]}),
$$
implies that the desired inequality $\Gamma(T;\ds) \leq \Gamma(T';\ds)$ holds under $\hat{h}$. That proves the claim.
\end{proof}

For the rest of the proof, we assume that $\hat{h} = \hat{h}_m$.
We prove the result by induction on the number of leaves.
The proof proceeds by considering the two subtrees 
hanging from the root in the hierarchy associated to the ultrametric $\ds$ and comparing their respective 
costs to that of the 
subtrees of any other hierarchy on the same sets of leaves. 

Let $\ds$ be an ultrametric dissimilarity on $L = [n]$. Let
$T$ be an associated hierarchy on $L$ with 
height function $h$. 
We start with the base of the induction argument.
\begin{claim}[Base case]
	If $n=2$, then $\Gamma(T;\ds) \leq \Gamma(T':\ds)$
	for any hierarchy $T'$ on $L$.
\end{claim}
\begin{proof}
	When $n=2$, there is only one hierarchy, so the statement
	is vacuous.
\end{proof}

Now suppose $n > 2$ and assume that the result holds
by induction for all hierarchies less than $n$ leaves.
Before we begin, it will
be convenient to define a notion of hierarchy 
allowing degree $2$ vertices.
\begin{definition}[Extended hierarchy]
	An \emph{extended hierarchy} on $L$ is a rooted tree $T'' = (V'',E'')$ with $n$ leaves, which we identify with the set $L$,
	such that all internal vertices have degree at most $3$ and the root has degree $2$. 
\end{definition}
We generalize the objective function to extended hierarchies $T''$
by letting
\begin{equation*}
\Gamma(T'';\ds)
= \sum_{s \in S_{T''}^2} 
\hat{h}(T''[s],\ds|_{L[s_-]\times L[s_+]}),
\end{equation*}
where $S_{T''}^2$ is the set of internal vertices of $T''$ \emph{with exactly
two immediate descendants}. 
That is, we ignore the degree $2$ vertices
in the objective, except for the root. We refer to these
ignored vertices as \emph{muted}.
We also trivially generalize to extended hierarchies 
the notion of an associated ultrametric.
The presence of degree $2$ vertices
will arise as a by-product of the following definition.
\begin{definition}[Restriction]
Let $T''$ be a hierarchy on $L''$ and
let $A'' \subseteq L''$. The restriction of $T''$
to $A''$, denoted $T''_{A''}$, is the extended hierarchy
obtained from $T''$ by keeping only those edges and vertices
lying on a path between two leaves in $A''$. 
\end{definition}
\noindent Note that applying the restriction procedure to a hierarchy
can indeed produce degree $2$ vertices and that the root of
a restriction has degree $2$ by definition.

We are now ready to proceed with the induction.
Let $\rho$ be the root of $T$
and
let $T_- = T[\rho_-]$,
$T_+ = T[\rho_+]$, $L_- = L_T[\rho_-]$ and $L_+ = L_T[\rho_+]$.
Let $T'$ be a distinct hierarchy on $L$.
Note that, for any subset $A \subseteq L$, the dissimilarity
$\ds|_{A\times A}$ is an ultrametric on $A$ as it continues to
satisfy condition~\eqref{eq:ultrametric}. 
Note, moreover, that $T_A$ is an extended hierarchy associated
with $\ds|_{A\times A}$, as the same heights can be used
on the restriction.
Hence, we can apply
the induction hypothesis to $L_-$ and $L_+$. That is, 
we have by induction that:
\begin{claim}[Induction on the subtrees hanging from the root]
\label{claim:induction-subtrees}
\begin{equation*}
\Gamma(T_-;\ds|_{L_-\times L_-})
\leq \Gamma(T'_{L_-};\ds|_{L_-\times L_-})
\quad\text{and}\quad
\Gamma(T_+;\ds|_{L_+\times L_+})
\leq \Gamma(T'_{L_+};\ds|_{L_+\times L_+}).
\end{equation*}
\end{claim}
\begin{proof}
In addition to the observations above,
we used 1) on the LHS of each inequality the fact that $T_- = T_{L_-}$ and
$T_+ = T_{L_+}$; and 2) on the RHS the fact that internal
vertices of degree $2$ (except the root) are ignored in the
objective (which is equivalent to suppressing those vertices
in the extended hierarchy and computing the objective over
the resulting hierarchy).
\end{proof}

We now relate the quantities in the previous claim 
to the objective values on $T$ and $T'$.
Let $\Delta := \max \ds$.
\begin{claim}[Relating $T$ and $T'$: Applying induction]
It holds that
\begin{equation*}
\Gamma(T;\ds)
\leq \Delta 
+ \Gamma(T'_{L_-};\ds|_{L_-\times L_-})
+ \Gamma(T'_{L_+};\ds|_{L_+\times L_+}).
\end{equation*}
\end{claim}
\begin{proof}
Because $T$ is associated to ultrametric $\ds$, 
the corresponding height of the root of $T$ is also
the largest height on $T$ and, hence,
\begin{equation*}
\hat{h}(T[\rho],\ds|_{L[\rho_-]\times L[\rho_+]})
=
\min \ds|_{L[\rho_-]\times L[\rho_+]}
=
\max \ds
= \Delta.
\end{equation*}  
Therefore, adding up the contributions to $\Gamma(T;\ds)$ of the root
and of the two subtrees hanging from the root, we get
\begin{equation*}
\Gamma(T;\ds)
= \Delta 
+ \Gamma(T_-;\ds|_{L_-\times L_-})
+ \Gamma(T_+;\ds|_{L_+\times L_+}).
\end{equation*}
We use Claim~\ref{claim:induction-subtrees} to conclude.
\end{proof}
So it remains to relate the RHS in the previous claim
to the objective value of $T'$. This involves a case analysis. We start with a simple case.
\begin{claim}[Relating $T$ and $T'$: Equality case]
	If there is $s \in S_{T'}$ such that $L_{T'}[s] = L_-$ or $L_+$, then it holds that
	\begin{equation}\label{eq:equality-case}
		\Gamma(T';\ds)
		= \Delta 
		+ \Gamma(T'_{L_-};\ds|_{L_-\times L_-})
		+ \Gamma(T'_{L_+};\ds|_{L_+\times L_+}).
	\end{equation}
\end{claim}
\begin{proof}
Observe that $s$ cannot be the root of $T'$ as otherwise we would have $L_- = \emptyset$. So $s$ has a parent. 
Let $\tilde{s}$ be the parent of $s$ with descendants $\tilde{s}_-$ and $\tilde{s}_+$, and assume without loss of generality that $L_{T'}[\tilde{s}_-] \subseteq L_-$ and $L_{T'}[\tilde{s}_+] = L_+$ (i.e., $\tilde{s}_+ = s$). Then the contribution to $\Gamma(T';\ds)$ of $\tilde{s}$ is $\Delta$. Furthermore, the
contribution to $\Gamma(T';\ds)$ of those vertices in $T'_{L_+}$ is $\Gamma(T'_{L_+};\ds|_{L_+\times L_+})$. Finally, in $T'_{L_-}$ vertex $s$ has degree 2 and so is muted. The remaining vertices of $T'_{L_-}$ contribute $\Gamma(T'_{L_-};\ds|_{L_-\times L_-})$ to both sides of~\eqref{eq:equality-case}. 
\end{proof}

The general case analysis follows.
We assume for the rest of the proof
that: 
\begin{equation}\label{eq:not-1-move}
\nexists s \in S_{T'}, L_{T'}[s] = L_- \text{ or } L_+.
\end{equation}
\begin{claim}[Relating $T$ and $T'$: Case analysis]
Under condition~\eqref{eq:not-1-move}, it holds that
\begin{equation}\label{eq:second-step}
\Gamma(T';\ds)
\geq \Delta 
+ \Gamma(T'_{L_-};\ds|_{L_-\times L_-})
+ \Gamma(T'_{L_+};\ds|_{L_+\times L_+}).
\end{equation}
\end{claim}
\begin{proof}
%Let
%$\rho'$ be the root of $T'$ and let $L'_- = L_{T'}[\rho'_-]$ and $L'_+ = L_{T'}[\rho'_+]$. If either $L'_- = L_-$ or $L'_- = L_+$ (that is, the roots of $T$ and $T'$ split the same subsets of leaves),
%then by the argument in the previous claim, \todo{Justify the $\Delta$.}
%\begin{equation*}
%\Gamma(T';\ds)
%= \Delta 
%+ \Gamma(T'_{L_-};\ds|_{L_-\times L_-})
%+ \Gamma(T'_{L_+};\ds|_{L_+\times L_+}),
%end{equation*}
%and we are done. So 
%we assume for the rest of the proof
%that, without loss of generality, 
%\begin{equation}\label{eq:root-overlap}
%L_-\cap L'_- \neq \emptyset
%\qquad
%\text{and} 
%\qquad
%L_-\cap L'_+ \neq \emptyset.
%\end{equation}
Recall that $\Gamma(T';\ds)$ is a sum over internal
vertices of $T'$. We divide up those vertices into
several classes. Below, we identify the vertices
in the restrictions to the original vertices and we write $s \in T''$ to indicate that $s$ is a vertex of $T''$.
Observe that, by definition, $T_- = T_{L_-}$ and
$T_+ = T_{L_+}$ do not share vertices---but that $T'_{L_-}$ and
$T'_{L_+}$ might. 
Recall that, for $s \in S_{T'}$, we denote by $s_-$ and $s_+$ the immediate descendants of $s$ in $T'$. 
%We use the notation $L_{T''}[s]$ to indicate the leaves below $s$ on $T''$ when the underlying hierarchy is not clear from context.
\begin{enumerate}
	\item \textbf{Appears in one subtree}: Let $R_1$ be the elements $s$ of $S_{T'}$ such that either (i) $s \in T'_{L_-}$ but
	$s \notin T'_{L_+}$, or (ii) $s \in T'_{L_+}$ but
	$s \notin T'_{L_-}$. It will be important below whether or not $s$ is muted. Case (i) means that there is a path on $T'$ between two leaves in $L_-$ that goes through $s$---but not between two leaves in $L_+$. Note that a path going through $s$ necessarily has an endpoint in $L_{T'}[s_-]$ or $L_{T'}[s_+]$, or both. We claim that, for such an $s$, we have that both $L_{T'}[s_-]$ and $L_{T'}[s_+]$ have a non-empty intersection with $L_-$. Indeed assume that, say, $L_{T'}[s_+]$ contains only leaves from $L_+$. Because there is no path between two leaves in $L_+$ going through $s$ in $T'$, it must be that actually $L_{T'}[s_+] = L_+$. But that contradicts condition~\eqref{eq:not-1-move}, and proves the claim. Moreover one of $L_{T'}[s_-]$ or $L_{T'}[s_+]$ (or both) is a subset of $L_-$, as otherwise there would be a path between two leaves in $L_+$ going through $s$ and we would have that $s \in T'_{L_+}$, a contradiction. In case (ii), the same holds with the roles of $L_-$ and $L_+$ interchanged. 
	That implies further that $s$ is not muted in the restriction it belongs to. However it is muted in the restriction it does not belong to. Let $r_1 = |R_1|$.
	
	\item \textbf{Appears in both, twice muted}: Let $R_{2,tm}$ be the elements $s$ of $S_{T'}$ such that $s \in T'_{L_-}$ and
	$s \in T'_{L_+}$ and $s$ is muted in both restrictions. That arises precisely when $L_{T'}[s_-]$ and $L_{T'}[s_+]$ each belong
	to a \emph{different} subset among $L_-$ and $L_+$. Let $r_{2,tm} = |R_{2,tm}|$.
	
	\item \textbf{Appears in both, once muted}: Let $R_{2,om}$ be the elements $s$ of $S_{T'}$ such that $s \in T'_{L_-}$ and
	$s \in T'_{L_+}$ and $s$ is muted in exactly one restriction. That arises precisely when one of $L_{T'}[s_-]$ and $L_{T'}[s_+]$ has a non-empty intersection with exactly \emph{one} of $L_-$ and $L_+$,
	and the other has a non-empty intersection with \emph{both} $L_-$ and $L_+$. Let $r_{2,om} = |R_{2,om}|$.
	
	\item \textbf{Appears in both, neither muted}: Let $R_{2,nm}$ be the elements $s$ of $S_{T'}$ such that $s \in T'_{L_-}$ and
	$s \in T'_{L_+}$ and $s$ is muted in neither restriction. That arises precisely when \emph{both} $L_{T'}[s_-]$ and $L_{T'}[s_+]$ have a non-empty intersection with \emph{both} $L_-$ and $L_+$. Let $r_{2,nm} = |R_{2,nm}|$.
	
\end{enumerate}
Because the sets above form a partition of $S_{T'}$ %(note in particular that the root of $T'$ appears in one of these sets by assumption~\eqref{eq:root-overlap}) 
and that any hierarchy on $n$ leaves has exactly $n-1$ internal vertices, it follows
that
\begin{equation*}
r_1 + r_{2,tm} + r_{2,om} + r_{2,nm} = n-1.
\end{equation*}
Moreover, on an extended hierarchy with $n' < n$ leaves, the number of internal non-muted vertices is $n'-1$ (which can be seen by collapsing the muted vertices). Hence, counting non-muted 
%and muted\todo{?} 
vertices on each restriction with multiplicity, we get the relation
\begin{equation*}
1\cdot r_1 + 0\cdot r_{2,tm} + 1\cdot r_{2,om} + 2 \cdot r_{2,nm}
= (|L_-| - 1) + (|L_+| - 1) = n-2. 
\end{equation*}
Combining the last two displays gives
\begin{equation}\label{eq:second-step-key}
r_{2,tm} = 1 + r_{2,nm}.
\end{equation}
This last equality is the key to comparing the two sides
of~\eqref{eq:second-step}: the twice muted vertices which 
contribute $\max \ds$ to the LHS are in one-to-one correspondence
with terms on the RHS whose contributions are smaller or equal.

We expand on this last point. 
To simplify the notation, we let
$\ds_- = \ds|_{L_-\times L_-}$
and
$\ds_+ = \ds|_{L_+\times L_+}$.
By the observations above,
we have the following. Recall that $\hat{h} = \hat{h}_m$.
\begin{enumerate}
	\item $R_1$: Each $s \in R_1$ is muted in the restriction it does not belong to but it is not in the restriction it belongs to, so that it contributed to exactly one term on the RHS, say $\Gamma(T'_{L_-};\ds|_{L_-\times L_-})$. In that case, we have shown that both $L_{T'}[s_-]$ and $L_{T'}[s_+]$ have a non-empty intersection with $L_-$. The RHS term $\hat{h}(T'_{L_-},\ds_-|_{L_{T'}[s_-]\times L_{T'}[s_+]})$
	differs from the corresponding LHS term $\hat{h}(T'[s],\ds|_{L_{T'}[s_-]\times L_{T'}[s_+]})$  only in that pairs $(x,y) \in L_{T'}[s_-]\times L_{T'}[s_+]$ with $(x,y) \in L_- \times L_+$ or $L_+\times L_-$ are removed (which we refer to below as being suppressed by restriction) from the minimum defining $\hat{h} = \hat{h}_m$---but such pairs contribute $\Delta = \max \delta$ and therefore do not affect the minimum on the LHS. We have also shown that none of these pairs can be in $L_+ \times L_+$. As a result, we have 
	\begin{eqnarray*}
		&&\sum_{s \in R_1} \hat{h}(T'[s],\ds|_{L_{T'}[s_-]\times L_{T'}[s_+]})\\
		&&= \sum_{s \in R_1\cap S_{T'_{L_-}}} \hat{h}(T'_{L_-},\ds_-|_{L_{T'}[s_-]\times L_{T'}[s_+]})
		+ \sum_{s \in R_1\cap S_{T'_{L_+}}} \hat{h}(T'_{L_+},\ds_+|_{L_{T'}[s_-]\times L_{T'}[s_+]}).
	\end{eqnarray*}
	
	\item $R_{2,tm}$: Each $s \in R_{2,tm}$ contributes to neither term on the RHS, as it is muted in both restriction. On the other hand, we have argued that $L_{T'}[s_-]$ and $L_{T'}[s_+]$ each belong
	to a \emph{different} subset among $L_-$ and $L_+$. Hence we have 
	\begin{eqnarray*}
	&&\sum_{s \in R_{2,tm}} \hat{h}(T'[s],\ds|_{L_{T'}[s_-]\times L_{T'}[s_+]}) = \Delta\cdot r_{2,tm},
	\end{eqnarray*}
	while
	\begin{eqnarray*}
	&&\sum_{s \in R_{2,tm}} \hat{h}(T'_{L_-},\ds_-|_{L_{T'}[s_-]\times L_{T'}[s_+]})
	+ \sum_{s \in R_{2,tm}} \hat{h}(T'_{L_+},\ds_+|_{L_{T'}[s_-]\times L_{T'}[s_+]}) = 0\cdot 2r_{2,tm}.
	\end{eqnarray*}
	
	\item $R_{2,om}$: In this case, we have that
	\begin{eqnarray*}
		&&\sum_{s \in R_{2,om}} \hat{h}(T'[s],\ds|_{L_{T'}[s_-]\times L_{T'}[s_+]})\\
		&&= \sum_{s \in R_{2,om}} \hat{h}(T'_{L_-},\ds_-|_{L_{T'}[s_-]\times L_{T'}[s_+]})
		+ \sum_{s \in R_{2,om}} \hat{h}(T'_{L_+},\ds_+|_{L_{T'}[s_-]\times L_{T'}[s_+]}),
	\end{eqnarray*}
	where we used that (1) each $s$ in $R_{2,om}$ is muted
	in exactly one of the sums on the second line and that (2) the pairs
	of leaves suppressed by restriction in the non-muted terms on the second line correspond to pairs on opposite sides of the root in $T$, which contribute $\Delta = \max \ds$ and therefore do not affect the minimum defining $\hat{h}$.
	
	\item $R_{2,nm}$: Each $s \in R_{2,nm}$ contributes to both terms $ \Gamma(T'_{L_-};\ds|_{L_-\times L_-})$
	and $\Gamma(T'_{L_+};\ds|_{L_+\times L_+})$ on the RHS of~\eqref{eq:second-step}, once with the same value as the corresponding term on the LHS and once with a larger value. Indeed, because both $L_{T'}[s_-]$ and $L_{T'}[s_+]$ have a non-empty intersection with both $L_-$ and $L_+$ and pairs
	$(x,y) \in L_- \times L_+$ or $L_+\times L_-$ 
	have dissimilarity $\Delta$, it follows that the minimum
	\begin{equation}\label{eq:r2nm-min}
	\hat{h}(T'[s],\ds|_{L_{T'}[s_-]\times L_{T'}[s_+]})
	= \min \ds|_{L_{T'}[s_-]\times L_{T'}[s_+]},
	\end{equation}
	is achieved for a pair $(x,y) \in L_- \times L_-$ or $L_+\times L_+$. The claim then follows by noticing that restriction increases the minimum. Let $R_{2,nm}^{-,=}$ be the set of all $s \in R_{2,nm}$ such that the minimum in~\eqref{eq:r2nm-min} is achieved for a pair in $L_- \times L_-$ and $R_{2,nm}^{+,=} = R_{2,nm} \setminus R_{2,nm}^{-,=}$. Then
	%\todo{Are these supposed to be $nm$'s? Should read the proof again and make sure it's understandable. Figures?}
	\begin{eqnarray*}
		&&\sum_{s \in R_{2,nm}} \hat{h}(T'[s],\ds|_{L_{T'}[s_-]\times L_{T'}[s_+]})\\
		&&= \sum_{s \in R_{2,nm}^{-,=}} \hat{h}(T'_{L_-},\ds_-|_{L_{T'}[s_-]\times L_{T'}[s_+]})
		+ \sum_{s \in R_{2,nm}^{+,=}} \hat{h}(T'_{L_+},\ds_+|_{L_{T'}[s_-]\times L_{T'}[s_+]}),
	\end{eqnarray*}
	while
	\begin{eqnarray*}
		&&\sum_{s \in R_{2,nm}^{+,=}} \hat{h}(T'_{L_-},\ds_-|_{L_{T'}[s_-]\times L_{T'}[s_+]})
		+ \sum_{s \in R_{2,nm}^{-,=}} \hat{h}(T'_{L_+},\ds_+|_{L_{T'}[s_-]\times L_{T'}[s_+]})\\
		&&\leq \Delta \cdot r_{2,nm}.
	\end{eqnarray*}

\end{enumerate}
To sum up, the contributions of $R_1$ and $R_{2,om}$
are the same on both sides of~\eqref{eq:second-step}. The contributions of $R_{2,nm}$ on the LHS are canceled out
by the contributions of $R_{2,nm}^{-,=}$ and $R_{2,nm}^{+,=}$ on the RHS. The remaining terms are:
on the LHS,
$\Delta \cdot r_{2,tm}$;
and on the RHS,
$\leq \Delta\cdot (1 + r_{2,nm})$. Using~\eqref{eq:second-step-key} concludes the proof.
\end{proof}
\noindent That concludes the induction and the proof of the theorem.
\end{proof}

%\todo{Add the strict case, plus example. Maybe a separate section (nobody will read this at this point)? Maybe in the examples section?}

%\todo{Mention scale invariance as another possibility, although not necessary for consistency? Mention the Carlsson paper and Kleinberg paper? Maybe modularity of cost as well (less clear how to impose that in our context)? See Dasgupta. That last one is convenient but not necessarily natural? Does the special case from Mathieu satisfy it? Yes, it seems so (only one?).}

%\todo{Mention that it's hard to prove high-probability results under phase transition? Maybe leave as open problem?}

%\todo{Are we implicitly constructing an ultrametric? Connection to Carlsson's work? Also making connection between the two different types of criteria?}

%\todo{Use the dendogram terminology (see Mathieu and also Carlsson)? It seems that one advantage of our approach is that it links directly to dendrograms. E.g. from Mathieu:  "Furthermore, as observed by Carlsson and Mémoli [15], many practical hierarchical clustering algorithms, such as the linkage-based algorithms, actually output a dendrogram equipped with a height function that corresponds to an ultrametric embed- ding of the data. While their work focuses on algorithms that find embeddings in ultrametrics, our work focuses on finding cluster trees. We remark that these problems are related but also quite different."}

\section*{Acknowledgments}

The author's work was supported by NSF grants DMS-1149312 (CAREER), DMS-1614242, CCF-1740707 (TRIPODS Phase I), DMS-1902892, DMS-1916378, and DMS-2023239 (TRIPODS Phase II), as well as a Simons Fellowship
and a Vilas Associates Award.
Part of this work was done at MSRI and the Simons Institute for the
Theory of Computing.
I thank Sanjoy Dasgupta, Varun Kanade, Harrison Rosenberg, 
Garvesh Raskutti and C\'ecile An\'e for
helpful comments.
%I also thank anonymous reviewers of a previous version of this manuscript for suggested improvements.

%\clearpage

%\section*{References}

%\small

\bibliographystyle{alpha}
\bibliography{bibtex}

\end{document}